%% file: main_paper.tex
\pdfoutput=1
\documentclass{article}

\usepackage[square,sort,comma,numbers]{natbib}
\usepackage[final]{bdl_2018}

\usepackage[utf8]{inputenc} 
\usepackage[T1]{fontenc}    
\usepackage{hyperref}       
\usepackage{url}            
\usepackage{booktabs}       
\usepackage{amsfonts}       
\usepackage{nicefrac}       
\usepackage{microtype}      

\newcommand{\RN}[1]{%
	\textup{\lowercase\expandafter{\it \romannumeral#1}}%
}

\input{math_commands}

\newcommand{\upmu}{\tilde{\mu}}

\newcommand{\slr}{\varepsilon}

\newcommand{\upz}{\tilde{z}}
\usepackage{color}

\title{Straight-Through Estimator as \\ Projected Wasserstein Gradient Flow}

\author{
Pengyu Cheng$^1$, Chang Liu$^2$, Chunyuan Li$^3$, \\
{\bf Dinghan Shen$^1$, Ricardo Henao$^1$ and Lawrence Carin$^1$} \\
{$^1$Duke University,~$^2$Tsinghua University,~$^3$Microsoft Research}\\
\texttt{pengyu.cheng@duke.edu}
}

\begin{document}
\maketitle
%

%

\begin{abstract}
The Straight-Through (ST) estimator is a widely used technique for back-propagating gradients through discrete random variables.
However, this effective method lacks theoretical justification.
In this paper, we show that ST can be interpreted as the simulation of the projected Wasserstein gradient flow (pWGF). 
Based upon this understanding, a theoretical foundation is established to justify the convergence properties of ST.
Further, another pWGF estimator variant is proposed, which exhibits superior performance on distributions with infinite support, \emph{e.g.}, Poisson distributions.
Empirically, we show that ST and our proposed estimator, while applied to different types of discrete structures (including both Bernoulli and Poisson latent variables), exhibit comparable or even better performances relative to other state-of-the-art methods. 
Our results uncover the origin of the widespread adoption of ST estimator, and represent a helpful step towards exploring alternative gradient estimators for discrete variables.

\end{abstract}

\section{Introduction}
Learning distributions in discrete domains is a fundamental problem  in machine learning.
This problem can be formulated in general as minimizing the following expected cost
\begin{align}\label{eq:obj}
    L(\vtheta)=\mathbb{E}_{\vz \sim p_\vtheta}[f(\vz)],
\end{align}
where $f(\vz)$ is the cost function, $\vz$ is a discrete (latent) random variable whose distribution $p_\vtheta$ is parameterized by $\vtheta$.
Typically, $\vtheta$ is obtained as the output of a Neural Network (NN), whose weights are learned by backpropagating the gradients through discrete random variables $\vz$.
%
%
%

In practice, direct gradient computations through the discrete random variables, $\nabla_{\vtheta}L(\vtheta) = \sum_{\vz} \nabla_{\vtheta}p_{\vtheta}(\vz) f(\vz)$ suffers from the curse of dimensionality, since it requires traversing through all possible joint configurations of the latent variable, whose number is exponentially large \wrt the latent dimension.
Due to this limitation, existing approaches resort to estimating the gradient $\nabla_\vtheta L(\vtheta)$ by approximating its expectation, where Monte Carlo sampling methods are typically employed.

The Straight-Through (ST) estimator \citep{hinton2012neural,bengio2013estimating}  is a widely applied method due to its simplicity and effectiveness. 
 The idea of ST is directly using the gradients of discrete samples as the gradients of the distribution parameters. Since discrete samples can be generated as the output of hard threshold functions with distribution parameters as input, Bengio et al \citep{bengio2013estimating} explain the ST estimator by set the gradients of hard threshold functions to $1$. 
 %
 However, this explanation lacks theoretical justification for the gradients of hard threshold functions. 


In this paper, we show that ST can be interpreted as simulating the projected Wasserstein gradient flow (pWGF) of a functional $F[\mu]:= \bbE_{\vz\sim \mu}[f(\vz)]$, where $\mu$ is a distribution in the target discrete distribution family with density $p_\vtheta$ parameterized by $\vtheta$.
Further, a more general optimizing scheme for \eqref{eq:obj} is introduced. Instead of directly updating $\mu$ in the discrete distribution family, $\mu$ is first updated to $\upmu$ on a larger Wasserstein distribution space where gradients are easier to compute.
Then, we project $\upmu$ back to the discrete distribution family $\calM$ as the updated distribution. Moreover, the projection follows the descending direction of $F[\cdot]$ in $\calM$, which justifies the effectiveness of ST.
This pWGF based updating scheme also motivates another variant that achieves faster convergence when the desired family of distributions has infinite support, \emph{e.g.}, Poisson.

\vspace{-2mm}
\section{Proposed Algorithm}
\vspace{-2mm}
Denote  $\calM= \{\mu :  \text{density of $\mu$ has the form of $p_\vtheta$}  \}$ as the $d$-dimensional discrete distributions family parameterized by $\vtheta$. With $F[\mu]:= \bbE_{\vz\sim \mu}[f(\vz)]$, the task \eqref{eq:obj} can be rewritten as
\begin{align} \label{obj-functional}
 \min_{\vtheta}\mathbb{E}_{\vz \sim p_\vtheta}[f(\vz)]=	\min_{\mu \in \calM} \bbE_{\vz \sim \mu} [f(\vz)] = \min_{\mu \in \calM}F[\mu],
\end{align}
where $f(\cdot)$ is assumed to be differentiable. To solve \eqref{obj-functional}, directly calculating the gradient $\nabla_{\calM} F$ is challenging, because the discrete distribution family $\calM$ is very restrictive on the gradients.
Alternatively, if we relax the discrete constraint and perform updates in an appropriate larger space $\tcalM$, the calculation of the gradient $\nabla_{\tcalM} F$ can be much easier. Therefore, as showed in Fig.~\ref{updating_algorithm}, in $k$-th updating iteration, we consider first updating the current distribution $\mu_k$ to $\upmu_k$ with stepsize $\varepsilon$ in a larger 2-Wasserstein space $\tcalM$ \citep{villani2008optimal}, then projecting $\upmu_k$ back to $\calM$ as updated discrete distribution $\mu_{k+1}$. Theorem \ref{converge-thm} in supplement guarantees that  our updating scheme converges with a small enough step size $\slr$. 


\begin{minipage}{\textwidth}
\begin{minipage}{0.48\textwidth}
	\centering
	\includegraphics[width= .8\linewidth]{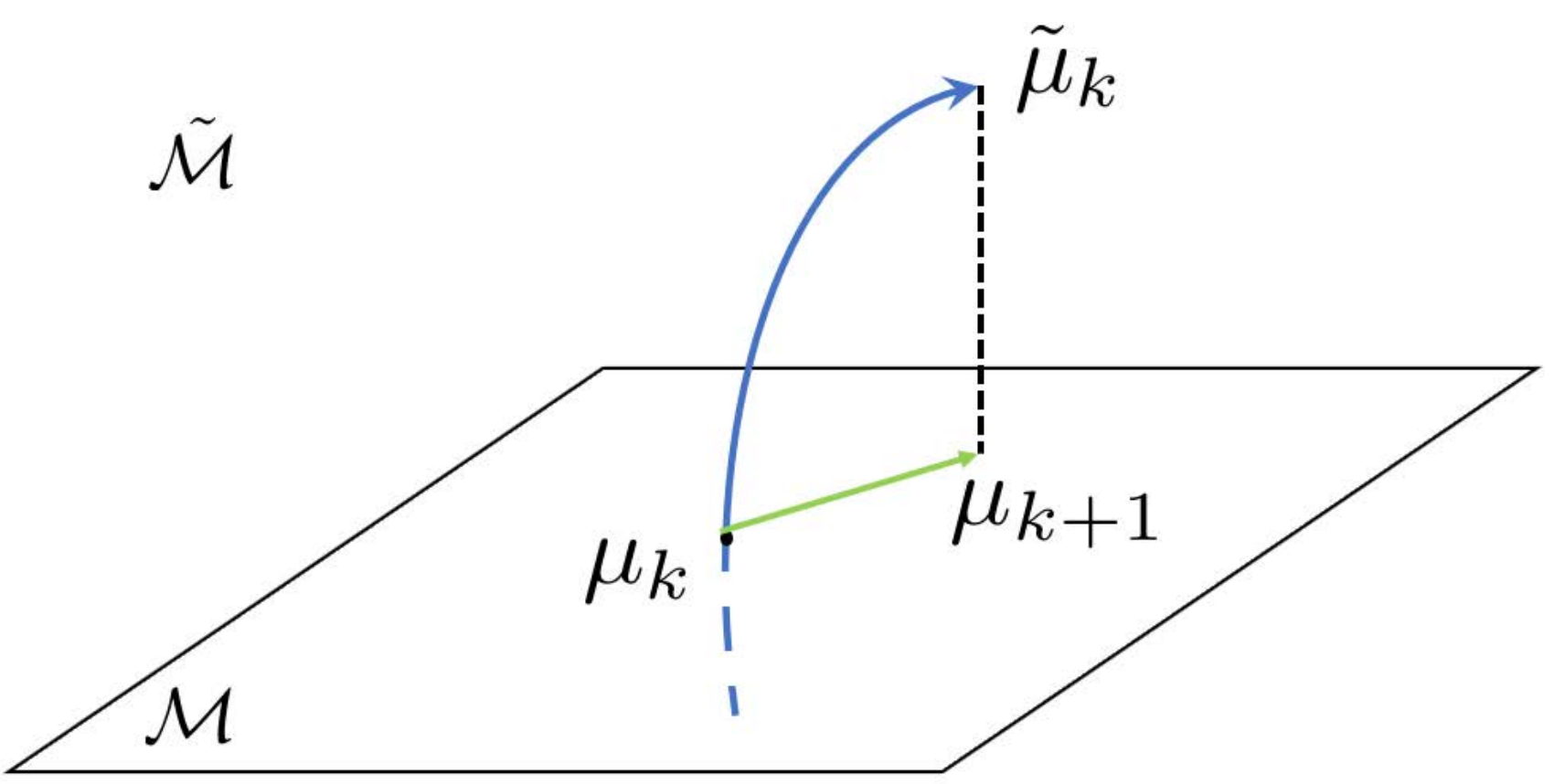}
    \captionof{figure}{Updating scheme}
	\label{updating_algorithm}
%
\hspace{-10mm}
\end{minipage}
\begin{minipage}{.45\textwidth}
	\centering
	\includegraphics[width=1.0\linewidth]{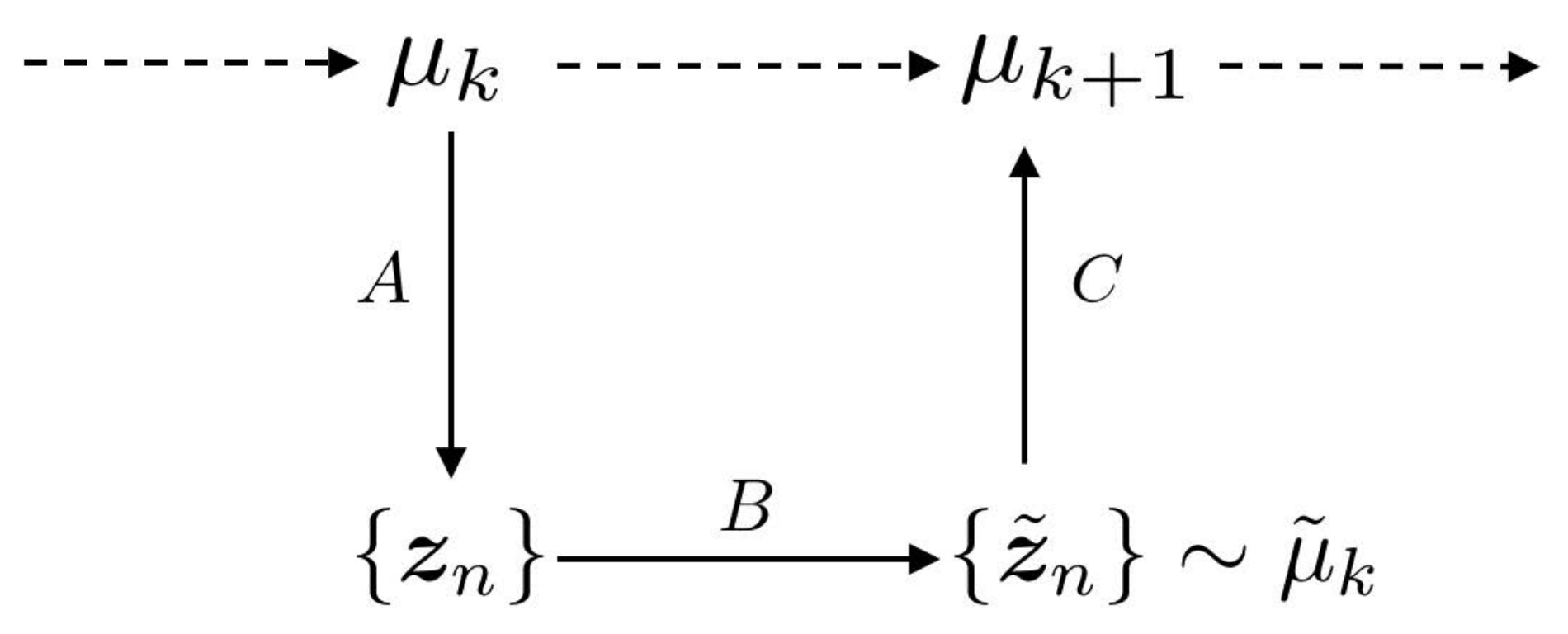}
	\captionof{figure}{Algorithm outline}
	\label{updating_sample}
\end{minipage}
\end{minipage}


With Wasserstein gradient flow (WGF) \citep{villani2008optimal}, we show (in Appendix) that, the gradient in larger space $\tcalM$ as $\nabla_{\tcalM} F = \nabla f$, which means, if $\mu_k$ is represented by a group of its samples $\{\vz_n\}_{n=1}^N$, then $\{\tilde{\vz}_n\}=\{\vz_n + \varepsilon \nabla f(\vz_n)\}$ can be treated as a group of sample from $\upmu_k$. Therefore, we can update $\mu_k$ to $\upmu_k$ along the WGF simply by updating its samples. 
%
%
%
To project $\upmu_k$ back to $\calM$ as $\mu_{k+1}$, we need to solve
$\mu_{k+1} = \argmin_{\mu \in \mathcal{M}} W(\mu,\upmu_k)$, which is equivalent to solve $ \argmin_{\mu}W^2(\mu,\upmu_k)$, where $W^2(\cdot,\cdot)$ is the square of the 2-Wasserstein distance \citep{villani2008optimal}. Consequently, our pWGF algorithm proceeds in 3 steps shown in Fig.~\ref{updating_sample}: (A) draw samples $\{\bm{z}_n\}$ from current distribution $\mu_k$ ; (B) update $\{\bm{z}_n\}$ to $\{\tilde{\bm{z}}_n\}$ as samples from $\upmu_k$; (C) project $\upmu_k$ back to $\mu_{k+1}$ by minimizing Wasserstein distance.

Since distributions in $\calM$ are multidimensional, the exact Wasserstein distance is difficult to derive.
We make a standard assumption~\citep{chen2016infogan} that $\mu$ and $\upmu_k$ are factorized distributions. With the assumption, we prove in Theorem~\ref{joint-eq-marginal} that minimizing Wasserstein distance between factorized distributions is equivalent to minimizing the marginal distance on every dimension. Therefore, for simplicity, we describe our projection step using one-dimensional distributions. As the updated distribution $\upmu_k$ is implicit, we can not obtain the closed form of Wasserstein distance $W^2(\mu_k$,$\tilde{\mu}_k)$.
Therefore, we consider two approximations of $W(\mu_k,\upmu_k)$. 
\begin{theorem}\label{joint-eq-marginal}
If $d$-dimensional distributions $\mu$ and $\nu$ are factorized, then $W^2(\mu,\nu) = \sum_{i=1}^d W^2(\mu^{(i)},\nu^{(i)})$, where $\mu^{(i)}$  and $\nu^{(i)}$ are the marginal distributions of $\mu$ and $\nu$ respectively.
\end{theorem}

\vspace{-2mm}
\subsection{ST estimator: Absolute Difference of Expectation}
\vspace{-2mm}
We find that the Straight-Through (ST) estimator~\citep{bengio2013estimating} is a special case of pWGF, when the Wasserstein distance is approximated via its lower bound, absolute difference of expectations.
\begin{theorem}\label{lower_bound_Wd}
For two one-dimensional distributions $\mu,\nu \in \tcalM$, the absolute difference between $\bbE_{\mu} = \bbE_{x \sim \mu} [x]$ and $\bbE_{\nu} = \bbE_{y \sim \nu}[y]$ is a lower bound of $W(\mu,\nu)$, i.e. $\left|\bbE_\mu - \bbE_\nu\right| \leq W(\mu,\nu).$
\end{theorem}
\begin{remark}\label{distance-bernoulli}
If $\mu$ and $\nu$ are Bernoulli,  then $W^2(\mu,\nu)=\left|\bbE_{ \mu} - \bbE_{\nu}\right|$, which means minimizing the expectation difference is equivalent to minimizing the 2-Wasserstein distance under Bernoulli cases. 
\end{remark}
For one-dimensional Bernoulli distribution, $\mu_k \sim \text{Bern}(p)$, noting that $p = \bbE_{\mu_k} \approx \frac{1}{N} \sum_{n=1}^N z_n$ and $\bbE_{\upmu_k} \approx \frac{1}{N} \sum_{n=1}^N \tilde{z}_n $, we approximate the parameter gradient by:
%
$	\nabla_p W^2({\mu _k},{\upmu_k}) \approx \nabla_p (\bbE_{z_k\sim\mu_k}[z_k]-\bbE_{\tilde{z}_k\sim\upmu_k}[\tilde{z}_k])^2 \notag 
\approx \nabla_p \left(p - \frac{1}{N} \sum_{n=1}^N \upz_n \right)^2 \notag 
 = 2(p - \frac{1}{N} \sum_{n=1}^N \upz_n).$ 
To reduce the variance caused by the sample mean, we use the control variate method  \citep{boyle1977options} and write
$
	\nabla_p W^2  \approx 2(p - \frac{1}{N} \sum_{n=1}^N \tilde{z}_n) 
	 = 2\left[(p-\bbE_{z_k\sim\mu_k}[z_k]) + (\bbE_{z_k\sim\mu_k}[z_k]-\frac{1}{N} \sum_{n=1}^N \tilde{z}_n) \right] 
	 \approx \frac{2}{N}\sum_{n=1}^N (z_n -\tilde{z}_n) = \frac{2 \slr}{ N} \sum_{n=1}^N \nabla_{z} f(z_n). $
Thus, we have derived the pWGF estimator with expectation difference approximation,
which has the same form as a multi-sample version ST estimator \citep{bengio2013estimating}. Parameter gradients for Poisson and Categorical distributions can be derived in a similar way.
%
%
\vspace{-2mm}
\subsection{Proposed estimator: Maximum Mean Discrepancy}
\vspace{-2mm}
A more principled way to approximate the Wasserstein distance is to use Maximum Mean Discrepancy (MMD) \citep{gretton2007kernel}:
$\Delta^2(\mu,\nu) = \bbE_{\vx_1,\vx_2 \sim \mu} [K(\vx_1,\vx_2)] + \bbE_{\vy_1,\vy_2 \sim \nu}[ K(\vy_1,\vy_2)]  -2 \bbE_{\vx \sim \mu, \vy \sim \nu} [ K(\vx,\vy)]$,
where $K(\cdot,\cdot)$ is a selected kernel.
In practice, instead of minimizing $W(\mu,\upmu_k)$, we can minimize the empirical expectation $\Delta^2(\mu,\upmu) \approx \bbE_{z_1,z_2 \sim \mu} [K(z_1,z_2)] + \frac{1}{N^2} \sum_{n,n'=1}^N K(\upz_n,\upz_{n'}) - 2\frac{1}{N} \sum_{n=1}^N \bbE_{z \sim \mu} K(z,\upz_n)$.
Details on parameter gradients $\nabla_\vtheta[{\Delta^2}]$ are shown in the supplement.
\vspace{-2mm}
\section{Experiments}
\vspace{-2mm}
We demonstrate the advantage of pWGF on updating Poisson distributions, and show the benchmark performance with a binary latent model in the supplement. 
Since the only difference between our pWGF version ST and the original ST is the learning rate scalar, if not specifically mentioned, we call pWGF-ST or the original ST together as ST, and call our MMD version method as pWGF.
%
%
%

\vspace{-2mm}
\subsection{Poisson Parameter Estimation}
We apply pWGF to infer the parameter of a one-dimensional Poisson distribution. We use the  true distribution $p(z) = \text{Pois}(\lambda_0=5)$ to generate data samples $\{z_i\}_{i=1}^N$, and use a Generative Adversarial learing framework to learn model parameters. A {\it generator} $q_\lambda(z)$ is constructed as $z \sim \text{Pois}(\lambda)$. A {\it discriminator} $w(z)$ is a network used to distinguish true/fake samples, which outputs the probability that the data comes from the true distribution.  
During the adversarial training, the generator aims to increase $\bbE_{z \sim q_\lambda} [w(z)]$, while the discriminator tries to decrease  $\bbE_{z \sim q_\lambda} [w(z)]$ and increase 
$\bbE_{z \sim p}[w(z)]$. We can rewrite the training process as a min-max game with objective function:
    $\max_{\lambda} \min_{w} \left\{ \bbE_{z \sim q_\lambda} [w(z)] -\bbE_{z \sim p}[w(z)] \right\}.$
Similar to the observation in \citep{goodfellow2014generative}, the training process should finally converges to $\lambda^* = \lambda_0 = 5$.
Therefore, for the generator, learning $\lambda$ becomes optimizing 
$\mathbb{E}_{z \sim q_\lambda}[w(z)]$.
%
%
We compare our pWGF against ST, Reinforce and Muprop~\citep{gu2015muprop} and show the learning curves on estimation in Figure~\ref{fig:Poisson_gan}. 
 pWGF converges faster than others and exhibits much smaller oscillation. In Table~\ref{mean-std}, We report the mean and the standard derivation of the inferred parameter $\lambda$ after $100$ training epochs, where our pWGF exhibits higher inference accuracy and lower variance. 
\begin{minipage}{\textwidth}
\begin{minipage}{0.6\textwidth}
	\centering
	\includegraphics[width=0.5\textwidth]{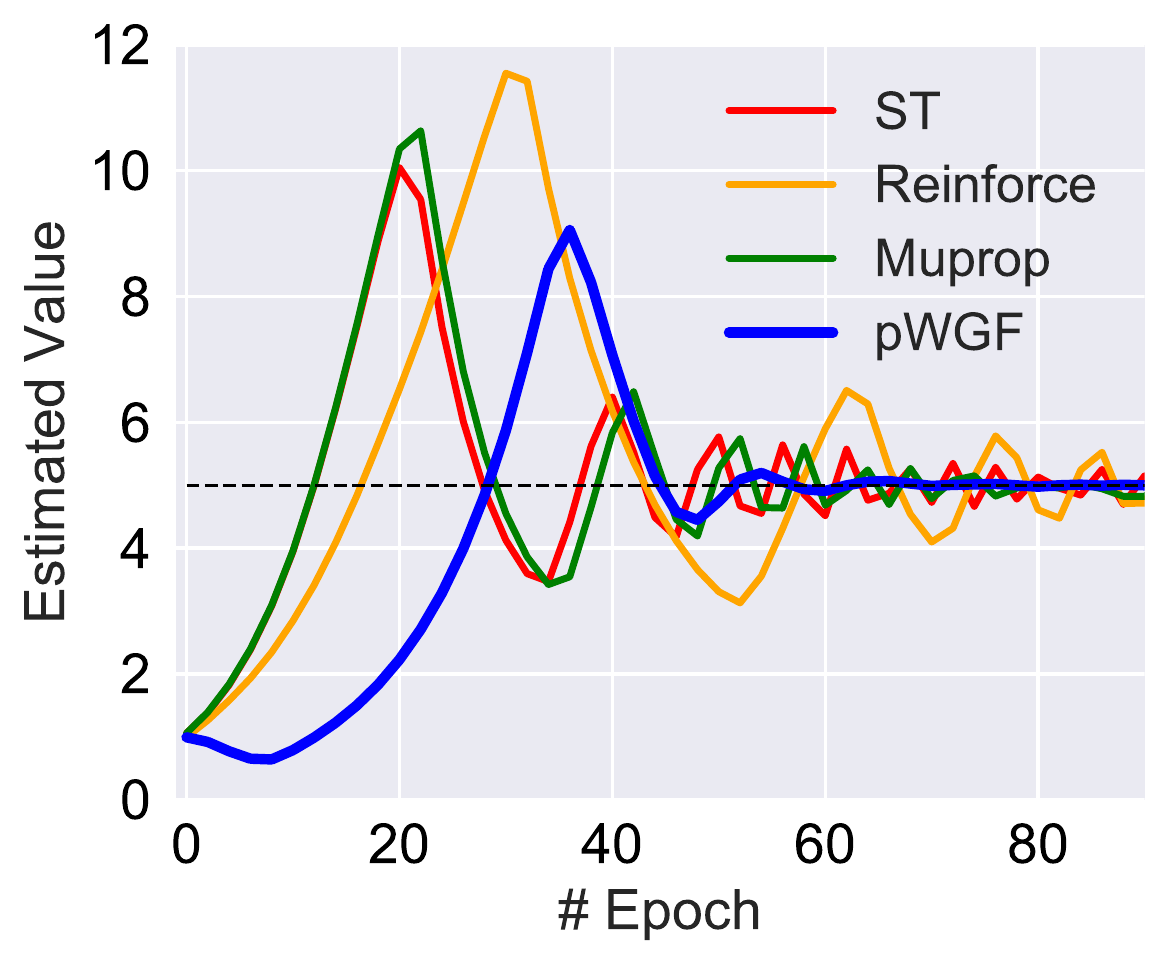}
	\vspace{-2mm}
	\captionof{figure}{Learning curves of Poisson parameter.}
	\label{fig:Poisson_gan}
\end{minipage}
\begin{minipage}{.28\textwidth}
  \centering
  \captionof{table}{Mean and Standard Derivation of Inference}
    \begin{tabular}{l|r|r}
          & \multicolumn{1}{l|}{Mean} & \multicolumn{1}{l}{Std} \\
    \hline
    pWGF &    5.0076   & 0.013 \\
    ST &     5.1049  &  0.161 \\
    Muprop &      5.0196 &  0.159\\
    Reinforce &    4.9452   & 0.173 \\
    \end{tabular}%
  \label{mean-std}%
\end{minipage}
\end{minipage}

\vspace{-2mm}
\section{Conclusion}
\vspace{-2mm}
We presented a theoretical foundation to justify the superior empirical performance of Straight-Through (ST) estimator for backpropagating gradients through discrete latent variables. Specifically, we show that ST can be interpreted as the simulation of the projected gradient flow on Wasserstein space. Based upon this theoretical framework, we further propose another gradient estimator for learning discrete variables, which exhibits even better performance while applied to distributions with infinite support, \emph{e.g.}, Poisson. 

\bibliographystyle{plain}
\bibliography{pWGF_reference}

\clearpage
\appendix

\section{Background}
%
%
To minimize the expected cost $\bbE_{\vz \sim p_\vtheta}[f(\vz)]$ in \eqref{eq:obj}, we assume that $\mathbb{E}_{\vz\sim p_\vtheta}[\nabla_\vtheta f(\vz)] = 0$, if the cost function $f(\vz)$ depends of $\vtheta$.
%
For instance, in the Variational Autoencoder (VAE) \citep{doersch2016tutorial}, we seek to maximize the Evidence Lower Bound (ELBO) as $\mathbb{E}_{\vz\sim q_\vtheta(\vz|\vx)}
[f(\vz)]$, where $f(\vz) = \log[p(\vx|\vz)p(\vz)/q_\vtheta(\vz|\vx)]$ depends on parameter $\vtheta$ through the variational posterior approximation $q_\vtheta(\vz|\vx)$.
Since $\mathbb{E}_{\vz\sim q_\vtheta(\vz|\vx)} [\nabla_\vtheta \log q_\vtheta(\vz|\vx)] = 0$, we have $\mathbb{E}_{\vz\sim q_\vtheta(\vz|\vx)}[\nabla_{\vtheta} f(\vz)] = 0$.

As described above, there are two types of updating methods for $\vtheta$ under \eqref{eq:obj}, namely, estimation of the parameter gradient $\nabla_\vtheta \bbE_{\vz\sim p_\vtheta}[f(\vz)]$, and continuous relaxation of the discrete variable $\vz$.

\subsection{Continuous relaxation}
Another approach used to obtain updates for $\vtheta$ in \eqref{eq:obj} is to approximate samples of $\vz$ from a deterministic function, $h(\cdot)$, of $\vtheta$ and an independent random variable $\bm{\epsilon}$ with simple distribution $p_{\bm{\epsilon}}$, \emph{e.g.}, uniform or normal, so $\vz=h(\vtheta,\bm{\epsilon})$.
Then we can use the chain rule to derive the gradient of \eqref{eq:obj} as
%
%
\begin{equation*}
	\nabla_{\vtheta} \bbE_{p_\vtheta}[f(\vz)] = \nabla_{\vtheta} \bbE_{p_{\bm{\epsilon}}} [f(h(\vtheta,\bm{\epsilon}))] = \bbE_{p_{\bm{\epsilon}}}[\nabla_\vtheta f(h(\vtheta,\bm{\epsilon}))].
\end{equation*}
 We can take expectation of the gradients, which is very convenient because $\nabla_\vtheta$ can be computed by chain rule, noting that $f(\cdot)$ does not directly depend of $\vtheta$. 
This reparameterization trick works quiet well when $\vz$ originates from a continuous distribution.
For example, given a normal distribution, $\vz\sim N(\vmu,{\rm diag}(\vsigma^2))$, we can rewrite $\vz = \vmu+{\rm diag}(\vsigma)N(0,\bm{I})$ and directly obtain $\nabla_{\vmu}{\vz}$ and $\nabla_{\vsigma}{\vz}$.
This reparameterization has been widely used in the training of variational autoencoder with latent Gaussian priors \citep{kingma2014auto,rezende2014stochastic}.

In the discrete case, it becomes very difficult to find a differentiable deterministic function to generate samples from $\vz$.
For the categorical distribution, \citep{jang2017categorical} introduced the Gumbel-Softmax distribution to relax the \emph{one-hot} vector encoding commonly used for categorical variables.
For the multidimensional (factorized) Bernoulli distribution with parameter $\vtheta=\vp$, the Straight Through (ST) estimator \citep{hinton2012neural,bengio2013estimating}, which considers the gradient of $N$ samples of $\vz$ directly, as the gradient of parameter $\nabla_\vtheta f$, can be also explained by setting the derivative $\nabla_{\vp}h$ of the discrete function $\vz=h(\vp,\bm{\epsilon})= \bm{1}_{\bm{\epsilon} > \vp}$ (coordinate-wise) directly to the identity matrix $\bm{I}$ \citep{bengio2013estimating}.

\subsection{Wasserstein gradient flow}
%
Wasserstein gradient flows (WGF)~\citep{chen2018unified,ambrosio2008gradient, villani2008optimal} have become popular in machine learning, due to its generality over parametric distribution families, and tractable computational efficiency.
The Wasserstein space is a metric space of distributions.
The WGF defines a family of steepest descending functions. It has been Bayesian inference, where the KL divergence of an approximating distribution to a target one is minimized by simulating its gradient flow.
\cite{chen2018unified} developed a unfnied framework to simulate the WGF, including Stein Variational Gradient Descent (SVGD)~\citep{liu2016stein,liu2017stein}
and Stochastic Gradient MCMC as its special cases.
\cite{chen2017particle} and  \cite{liu2018accelerated} proposed an acceleration framework for these methods.
WFGs have also been applied to deep generative models~\citep{chen2017continuous} and policy optimization in reinforcement learning~\citep{zhang2018policy}. However, all previous methods focus on simulating WGFs to approximate distributions in continuous domains. There has been little if any research reported for WGFs for discrete domains.


\section{Updating via Wasserstein gradient flow}\label{sc:wgf}
Gradient computation and Wasserstein Gradient Flow (WGF) simulation are made possible by the Riemannian structure of $\tcalM$, which consists of a proper inner product in the tangent space that is consistent with the Wasserstein distance~\citep{benamou2000computational, otto2001geometry}.
The tangent space of $\tcalM$ at $\mu$ can be represented by a subspace of vector fields on $\bbR^d$ (\cite{villani2008optimal}, Thm 13.8; \cite{ambrosio2008gradient}, Thm 8.3.1, Prop 8.4.5):
$$T_{\mu}\tcalM := \overline{\{\nabla\varphi: \varphi \in C_c^{\infty}(\bbR^d)\}}^{L^2(\mu;\bbR^d)},$$
where $\calC_c^{\infty}(\bbR^d)$ is the set of compactly supported smooth functions on $\bbR^d$, $L^2(\mu;\bbR^d): = \{v : \int_{\bbR^d} v(\vz)\trs v(\vz) \mu(\ud \vz) < +\infty\}$ is a Hilbert space with inner product $\langle v, u \rangle_{L^2(\mu;\bbR^d)}:= \int v(\vz)\trs u(\vz) \mu(\ud \vz)$, and the overline represents taking the closure in ${L^2(\mu;\bbR^d)}$.

With the inner product inherited from $L^2(\mu;\bbR^d)$, $\tcalM$ being a Riemannian manifold is consistent with the Wasserstein distance due to the Benamou-Brenier formula~\citep{benamou2000computational}.
We can then express the gradient of a function on $\tcalM$ in the Riemannian sense.
The explicit expression is intuitively proposed as Otto's calculus~(\cite{otto2001geometry}; \cite{villani2008optimal}, Chapter~15) and rigorously verified by subsequent work, \emph{e.g.}, \cite{villani2008optimal}, Thm 23.18; \cite{ambrosio2008gradient}, Lem 10.4.1.
Specifically, they showed that given a functional $F[\mu]= \bbE_{\vz\sim\mu}[f(\vz)]$ with $f(\cdot)\in\calC^{\infty}_c(\bbR^d)$, its gradient is $\nabla_{\tcalM} F[\mu] = \nabla f \in T_{\mu}\tcalM$, a vector field on $\bbR^d$.
This means that we can, in principle, compute the desired gradient $\nabla_{\tcalM} F[\mu]$ using  $\nabla f$.

Another convenient property of $\tcalM$ based on the physical interpretation of tangent vectors on $\tcalM$ makes the gradient flow simulation possible.
Consider a smooth curve of absolutely continuous measures, $\mu_{t}$, with corresponding tangent vector $\vv_{t}$, where $t\in\bbR$, and for which the gradient flow is simulated (iteratively) at discrete values $k=1,\ldots,k,k+1,\ldots$, to estimate $\mu_1,\ldots,\mu_k,\mu_{k+1},\ldots$ (the target distribution).
For any $s\in\bbR$ and $\varepsilon\to 0$, Proposition 8.4.6 of~\cite{ambrosio2008gradient} guarantees that $W(\mu_{s+\varepsilon}, (\id + \varepsilon \vv_s)_{\#}\mu_s) = o(|\varepsilon|)$, where $(\id+\varepsilon \vv_s)$ is a transformation on $\bbR^d$ ($\id$ is the identity map and $\vv_s$ is a vector field on $\bbR^d$), and $(\id + \varepsilon \vv_s)_{\#}\mu_s$ is the pushed-forward measure of $\mu_s$ that moves $\mu_s$ along the tangent vector $\vv_s$ by distance $\varepsilon$, see Figure \ref{updating_algorithm}.
When $\mu_t$ is a gradient flow (steepest descending curve) of $F[\cdot]$ defined in the form above, $\vv_t = -\nabla_{\tcalM} F[\mu_t] = -\nabla f$, as described before, then for $\mu_k := \mu_s$ having a set of samples $\{\vz_n\}_{n=1}^N$ and the definition of pushed-forward measure \citep{ambrosio2008gradient}, $\{\tilde{\vz}_n := \vz_n - \slr\nabla f(\vz_n)\}_{n=1}^N$ is a set of samples of $\upmu_k := (\id + \slr \vv_s)_{\#}\mu_s$, which conform a first-order approximation of $\tilde{\mu}_{s+\slr}$.
Since $\tilde{\mu}_{k}\in\tcalM$ is a good approximation of $\mu_{s+\varepsilon}\in\tcalM$ (the optimal measure along the WGF) as discussed above, thus we can use $\tilde{\mu}_k\in\tcalM$ to approximate $\mu_{k+1}\in\calM$. This is done by projecting $\tilde{\mu}_k\in\tcalM$ onto $\mu_{k+1}\in\calM$.
Then, per Theorem \ref{converge-thm}, with small enough positive $\slr$, we can always get a set of samples whose distribution improves $F[\cdot]$, the functional of the cost in \eqref{eq:obj}.

\section{Proofs}
\begin{theorem}\label{theoretical-projection}
Let $F[\cdot]$ be a differentiable function on a manifold $\tcalM$ and $\calM$ a submanifold of $\tilde{\calM}$, $\calM \subset \tcalM$, then at any $\mu \in \calM$,
\begin{align*}
	\nabla_\calM F = (\nabla_{\tilde{\calM}} F)^{\perp},
\end{align*}
where $(\nabla_{\tilde{\calM}} F)^{\perp}$ is the projection of $ \nabla_{\tilde{\calM}} F$ onto $T_\mu \calM$. \\
\end{theorem}

\begin{proof}[Proof of Theorem \ref{theoretical-projection}]
By the definition of $\nabla_{\calM} F$ \citep{mukherjee2010learning}, for any vector $\bm{v} \in T_\mu {\calM}$, 
\begin{nalign}\label{thm1-eq1}
    \langle \nabla_{\calM} F , \bm{v} \rangle = \bm{v}(F)[\mu].
\end{nalign}
By the definition of $\nabla_{\tcalM} F$, for any vector $\bm{u} \in T_\mu {\tcalM}$, 
\begin{nalign}\label{thm1-eq2}
    \langle \nabla_{\tcalM} F , \bm{u} \rangle = \bm{u}(F)[\mu].
\end{nalign}
Since $T_\mu \calM$ is the subspace of $T_\mu \tilde{\calM}$, by definition of $(\nabla_{\tcalM} F)^{\perp}$ we have \begin{nalign}\label{thm1-eq3}
    \langle \nabla_{\tcalM} F ,\vv \rangle = \langle  (\nabla_{\tcalM} F)^{\perp}, \vv  \rangle.
\end{nalign}

By \eqref{thm1-eq1}, \eqref{thm1-eq2}, \eqref{thm1-eq3}, for any $\vv \in T_\mu {\calM}$
\begin{equation*}
\langle \nabla_{\calM} F , \bm{v} \rangle = \bm{v}(F)[\mu] =  \langle \nabla_{\tcalM} F ,\vv \rangle = \langle  (\nabla_{\tcalM} F)^{\perp}, \vv  \rangle.
\end{equation*}
 Therefore, $\nabla_\calM F = (\nabla_{\tcalM} F)^{\perp}$. 
\end{proof}

\begin{theorem} \label{converge-thm}
Let $\vv=-\slr \nabla_{\tcalM} F [\mu_k]$ and $W(\cdot,\cdot)$ be the 2-Wasserstein distance in $\tcalM$.
Update $\mu_k$ in $\tcalM$ along direction $\vv$ to $\upmu_{k} = \exp_{\mu_k}(\vv)$ (exponential map \citep{mukherjee2010learning}), then project $\upmu_k$ back to $\calM$ as $\mu_{k+1} = \argmin_{\mu \in \calM} W(\mu,\upmu_{k})$.
If  $\nabla_{\tcalM} F$  is Lipschitz continuous, then there exists $r>0$, such that for any $\slr<r$, 
	$F[\mu_k] \geq F[\mu_{k+1}]+ O(\slr^2)$.

\end{theorem}








\begin{proof}[Proof of Theorem \ref{joint-eq-marginal}]

  {(1)} First, we show that $W^2(\mu,\nu)\leq \sum_{i=1}^d W^2(\mu_i,\nu_i)$.
  
  Arbitrarily selecting  $\gamma_i \in \Gamma(\mu_i,\nu_i)$, $i=1,\dots,d$, we define 
  $\gamma^* = \prod_{i=1}^d \gamma_i$.  Since $\mu_i(x)= \int \gamma_i(x,\ud y)$, we have
  \begin{nalign}
  \begin{aligned}
      &\int_{\bbR^d} \gamma^* (x_1,\dots,x_d,\ud y_1,\dots,\ud y_d) \\
      =& \int_{\bbR^d} \gamma_1(x_1,\ud y_1)\gamma_2(x_2,\ud y_2)\dots  \gamma_D(x_d,\ud y_d) \\
      =& \prod_{i=1}^d \int_\bbR \gamma_i(x_i,\ud y_i) 
      = \prod_{i=1}^d \mu_i(x_i) = \mu(x_1,x_2,\dots,x_d),
  \end{aligned}
  \end{nalign}
  which means the marginal distribution of $\gamma^*$ on $\vx$ is $\mu$. Similarly, the marginal distribution of $\gamma^*$ on $\vy$ is $\nu$. Therofore, $\gamma^* \in \Gamma(\mu,\nu)$.   Then  
  \begin{nalign}\label{marginal_1}
   \inf_{\gamma \in \Gamma(\mu,\nu)} \int \Vert \bm{x}-\bm{y}\Vert^2 \gamma(\ud \bm{x},\ud \bm{y}) \leq \int \Vert \bm{x}-\bm{y} \Vert^2 \gamma^*(\ud \bm{x},\ud \bm{y}).
  \end{nalign}
  
  On the other hand, 
  \begin{nalign}\label{marginal_2}
  \begin{aligned}
    &  \int  \Vert \bm{x}-\bm{y} \Vert^2 \gamma^*(\ud \bm{x},\ud \bm{y}) \\
    =& \int \sum_{i=1}^d (x_i-y_i)^2 \gamma_1(\ud x_1,\ud y_1) \gamma_2(\ud x_2,\ud y_2) \dots \gamma_D (\ud x_d,\ud y_d) \\
    =& \sum_{i=1}^d  \int (x_i-y_i)^2 \gamma_i(\ud x_i,\ud y_i).
  \end{aligned}
  \end{nalign}
  By \eqref{marginal_1} and \eqref{marginal_2}, we have 
  \begin{nalign}\label{marginal_3}
      W^2(\mu,\nu) \leq \sum_{i=1}^d  \int (x_i-y_i)^2 \gamma_i(\ud x_i,\ud y_i).
  \end{nalign}
  Take the  infimum over both sides of the equation \eqref{marginal_3}, 
  \begin{nalign}
      W^2(\mu,\nu) \leq \sum_{i=1}^d \inf_{\gamma_i \in \Gamma(\mu_i,\nu_i)} \int (x_i-y_i)^2 \gamma_i(\ud x_i,\ud y_i) = \sum_{i=1}^d W^2(\mu_i,\nu_i).
  \end{nalign}

{(2)}  Then we  show $W^2(\mu,\nu) \geq \sum_{i=1}^d W^2(\mu_i,\nu_i).$
  
  Note that  
  \begin{nalign}\label{marginal_4}
  \begin{aligned}
     & \int \Vert \bm{x} - \bm{y}\Vert^2 \gamma(\ud \bm{x},\ud \bm{y}) \\
     =& \int \sum_{i=1}^d (x_i-y_i)^2 \gamma(\ud x_1,\dots,\ud x_d,\ud y_1,\dots,\ud y_d)\\
    = & \sum_{i=1}^d  \int (x_i-y_i)^2 \hat{\gamma}_i (\ud x_i,\ud y_i) ,
   \end{aligned}    
  \end{nalign}
  where $\hat{\gamma}_i(x_i,y_i) = \int \gamma(\ud x_1,\dots,\ud x_{i-1},x_i,\ud x_{i+1},\dots,\ud y_{i-1},y_i,\ud y_{i+1},\dots,\ud y_d)  $ is the marginal distribution of $\gamma$ over $(x_i,y_i)$.

By Fubini's Theorem,
\begin{nalign}
\begin{aligned}
    &\int \hat{\gamma}_i(x_i,\ud y_i)\\
    =&  \int \gamma(\ud x_1,\dots,\ud x_{i-1},x_i,\ud x_{i+1},\dots,\ud y_{i-1},\ud y_i,\ud y_{i+1},\dots,\ud y_d)\\
    =& \int \mu(\ud x_1,\dots,\ud x_{i-1},x_i,\ud x_{i+1},\dots \ud x_d) \ \ \ \    \text{(by $\mu(\bm{x}) = \int \gamma(\bm{x} ,\ud \bm{y})$)} \\
    =& \mu_i(x_i).
\end{aligned}
\end{nalign}
Similarly, $\int \hat{\gamma}_i (\ud x_i,y_i) = \nu_i(y_i)$. Therefore, $\hat{\gamma}_i \in \Gamma(\mu_i,\nu_i)$. Then
\begin{nalign} \label{marginal_5}
   \sum_{i=1}^d  \int (x_i-y_i)^2 \hat{\gamma}_i (\ud x_i,\ud y_i) \geq \sum_{i=1}^d \inf_{\gamma_i \in \Gamma(\mu_i,\nu_i)} \int (x_i-y_i)^2 {\gamma}_i (\ud x_i,\ud y_i).
\end{nalign}
By \eqref{marginal_4} and \eqref{marginal_5},
\begin{nalign}
 \int \Vert \bm{x} - \bm{y}\Vert^2 \gamma(\ud \bm{x},\ud \bm{y})\geq  \sum_{i=1}^d  W^2(\mu_i,\nu_i).
\end{nalign}
Take infimum over both sides, $
    W^2(\mu,\nu) \geq \sum_{i=1}^d W^2(\mu_i,\nu_i).$

Therefore, $W^2(\mu,\nu) =\sum_{i=1}^d W^2(\mu_i,\nu_i) $.
\end{proof}

\begin{proof}[Proof of Remark~\ref{lower_bound_Wd}] 
    For $\mu = \text{Bern}(p)$ and $\nu = \text{Bern}(q)$, 
    \begin{nalign}\label{w_bern}
    \begin{aligned}
    W^2(\mu,\nu) &= \inf_{\{a_{i,j}\}} \sum_{i,j \in \{0,1\}} a_{i,j} (i-j)^2 \\
               &= \min_{\{a_{i,j}\}} a_{1,0}+a_{0,1},
    \end{aligned}
    \end{nalign}
where $\sum_i a_{i,1}=q$, $\sum_j a_{1,j} = p$, and $a_{i,j} \geq 0$, $\sum_{i,j} a_{i,j} =1$. 

Problem in \eqref{w_bern} is a linear programming. It can be shown easily that the minimum value of \eqref{w_bern} is $W^2(\mu,\nu)= |p-q|$. 
\end{proof}

\begin{lemma} \label{bernoulli-other-distance}
 Let $\nu$ be an arbitrary distribution and $\mu= \text{Bern}(p)$ be a Bernoulli distribution. Then 
\begin{nalign}
W^2(\mu,\nu) = \int_{-\infty}^{t^*} y^2 \nu(\ud y) + \int_{t^*}^\infty (y-1)^2 \nu(\ud y),
\end{nalign}
where $t^* = \inf \{t: \int_t^\infty \nu(\ud y) = p\}$.
\end{lemma}

\section{Gradient For MMD Projection}
We take the radial basis function kernel $K(x,y) = \exp(-\frac{(x-y)^2}{2h^2})$ for instance. 

For Bernoulli distribution, $\mu =\text{ Bern}(p)$, $\frac{\partial \Delta^2}{\partial p} = 2[(1-2p)(1-K(1,0))-\frac{1}{n}\sum_{i=1}^n (K(1,\upz_i) -K(1,\upz_j))] $

\section{Binary Latent Models}
As most of previous proposed algorithms are specifically designed for the discrete variables with finite support, we consider using a binary latent model as the benchmark.
We use variational autoencoder (VAE) \citep{kingma2014auto} with the Bernoulli latent variable (Bernoulli VAE).
We compare pWGF with the baseline methods ST and Gumbel-Softmax~\citep{jang2017categorical}, as well as three state-of-the-art algorithms: Rebar~\citep{tucker2017rebar}, Relax~\citep{grathwohl2017backpropagation} and ARM~\citep{yin2018arm}.
Following the settings in \citep{yin2018arm}, we build the model with different network architectures. 
We apply all methods and architectures to the MNIST dataset, and show the results in Table \ref{bern-vae}.
From the results, pWGF is comparable with ST, and
both pWGF/ST outperform  other competing methods except ARM in all tested network architecture.


\begin{table}[htbp]
  \centering
  \caption{Testing ELBO for Bernoulli VAE on MNIST}
    \begin{tabular}{r|rrrrrr}
          & pWGF  & ST    & ARM   & RELAX & REBAR & Gumbel-Softmax \\
    \hline
    Linear & 119.8 & 119.1 & 110.3 & 122.1 & 123.2 & 129.2 \\
    Two Layers & 108.3 & 107.6 & 98.2  & 114   & 113.7 & NA \\
    Nonlinear & 104.6 & 104.2 & 101.3 & 110.9 & 111.6 & 112.5 \\
    \end{tabular}%
  \label{bern-vae}%
\end{table}%



\end{document}

%% file: math_commands.tex
\usepackage{amsmath,amsthm,amsfonts,amssymb}
\usepackage{bm}
\usepackage{algorithm, algorithmic}
\usepackage{graphicx}
\usepackage{sidecap}
\usepackage{subcaption, wrapfig}
\usepackage{mathrsfs}
\usepackage{multirow, multicol}

\newtheorem{theorem}{Theorem}[section]
\newtheorem{lemma}[theorem]{Lemma}
\newtheorem*{remark}{Remark}

\newenvironment{nalign}{
    \begin{equation}
    \begin{aligned}
}{
    \end{aligned}
    \end{equation}
}

\newcommand{\bbR}{\mathbb{R}}

\newcommand{\bbE}{\mathbb{E}}

\newcommand{\calC}{\mathcal{C}}

\newcommand{\calM}{\mathcal{M}}

\newcommand{\tcalM}{\tilde{\calM}}

\newcommand{\ud}{\mathrm{d}}

\newcommand{\id}{\mathrm{id}}

\newcommand{\trs}{^{\top}}
\newcommand{\wrt}{w.r.t.~}

\def\1{\bm{1}}

\def\vmu{{\bm{\mu}}}
\def\vtheta{{\bm{\theta}}}
\def\vsigma{{\bm{\sigma}}}

\def\vp{{\bm{p}}}

\def\vv{{\bm{v}}}

\def\vx{{\bm{x}}}
\def\vy{{\bm{y}}}
\def\vz{{\bm{z}}}

\DeclareMathAlphabet{\mathsfit}{\encodingdefault}{\sfdefault}{m}{sl}
\SetMathAlphabet{\mathsfit}{bold}{\encodingdefault}{\sfdefault}{bx}{n}

\DeclareMathOperator*{\argmin}{arg\,min}